\documentclass[12pt]{interact}

\usepackage{epstopdf}% To incorporate .eps illustrations using PDFLaTeX, etc.
\usepackage[caption=false]{subfig}% Support for small, `sub' figures and tables

\usepackage[numbers,sort&compress]{natbib}% Citation support using natbib.sty
\bibpunct[, ]{[}{]}{,}{n}{,}{,}% Citation support using natbib.sty
\makeatletter% @ becomes a letter
\def\NAT@def@citea{\def\@citea{\NAT@separator}}% Suppress spaces between citations using natbib.sty
\makeatother% @ becomes a symbol again

\usepackage{amssymb, amsmath, amsthm}
\usepackage[noend]{algpseudocode}
\usepackage{algorithmicx,algorithm}
\usepackage{graphicx}

\theoremstyle{plain}% Theorem-like structures provided by amsthm.sty
\newtheorem{theorem}{Theorem}[section]
\newtheorem{lemma}[theorem]{Lemma}

\newtheorem{proposition}[theorem]{Proposition}

\theoremstyle{definition}

\theoremstyle{remark}
\newtheorem{remark}{Remark}

\begin{document}

\articletype{}% Specify the article type or omit as appropriate

\title{Fast Deflation Sparse Principal Component Analysis \\via Subspace Projections}

\author{
\name{Cong Xu\textsuperscript{a},
           Min Yang\textsuperscript{a}\thanks{CONTACT Min Yang. Email: yang@ytu.edu.cn}
           Jin Zhang\textsuperscript{b}}
\affil{\textsuperscript{a}School of Mathematics and Information Sciences, Yantai University, Yantai, China;
         \textsuperscript{b}Department of Mathematics, Shandong Normal University, Jinan, China}
}

\maketitle

\begin{abstract}
The implementation of conventional sparse principal component analysis (SPCA)
on high-dimensional data sets has become a time consuming work.
In this paper, a series of subspace projections are constructed efficiently by using Household QR factorization.
With the aid of these subspace projections,
a fast deflation method, called SPCA-SP, is developed for SPCA.
This method keeps a good tradeoff between various criteria,
including sparsity, orthogonality, explained variance, balance of sparsity, and computational cost.
Comparative experiments on the benchmark data sets confirm the effectiveness of the proposed method.
\end{abstract}


\section{Introduction}
Principal component analysis (PCA) \cite{Bro2014PCA,Jolliffe1986pca, prei1988}
is a traditional and widely used tool for data processing and dimensionality reduction \cite{abra2014,Aitsa2019,bou2014,dhi2017,Fu2016,han2006,hron2016simplicial,Saad1998}.
Given a data set, PCA aims at finding  a sequence of orthogonal vectors that represent the directions of largest variance.
By capturing these directions, the principal components offer a way to compress the data with minimum information loss.
However, principal components are usually linear combinations of all original features.
That is, the weights in the linear combinations (known as loadings) are typically non-zero.
In this sense, it is difficult to give a good physical interpretation.

\smallskip
During the past decade, various sparse principal component analysis (SPCA) approaches
have been developed to improve the interpretability of principal components.
SPCA is an extension of PCA that aims at finding sparse loading vectors capturing the maximum amount of variance in the data.
These SPCA methods can be categorized into two groups:
block methods \cite{hu2016sparse,jolliffe2003a,kawano2015sparse, kawano2018sparse,ma2013,zou2006sparse}
and deflation methods \cite{daspremont2007a,daspremont2007full,mackey2009deflation,shen2008sparse}.
Block methods aims to find all sparse loadings together,
while deflation methods compute one loading at a time.
For examples,
Zou et al. \cite{zou2006sparse}
formulated sparse PCA as a regression-type optimization problem by imposing the
LASSO penalty on the regression coefficients.
Mackey \cite{mackey2009deflation} considered several deflation approaches
to explicitly maximize the additional variance under certain cardinality constraint.
Journee et al. \cite{journee2010generalized} developed the generalized power method,
in which sparse PCA is formulated as two single-unit and two block optimization problems.
Yuan et al. \cite{yuan2013truncated} proposed a fast SPCA method by combining the power method with truncation operation.
Recently, Hu et al. \cite{hu2016sparse} studied several SPCA algorithms via rotation and truncation.
Most existing SPCA methods deal with the original data directly.
However, due to the growing ease of observing variables,
high-dimensional data become more and more common,
which makes traditional methods very time consuming.

\smallskip
In recent years, a number of randomized methods, e.g. \cite{abra2014,Hal2011}, have been developed to enable fast PCA.
These methods first utilize random subspace projection to generate a small matrix
that captures the most explained variance in the original data.
Then standard SVD or eigen-decomposition is performed on this reduced matrix.
Nevertheless, it is still not clear how to extend such technique to SPCA.
Different from PCA,
in the framework of SPCA,
except explained variance,
there are more criteria like sparsity and orthogonality that need special attention.
A desirable fast SPCA should keep a good tradeoff between various criteria,
including sparsity, orthogonality, explained variance, balance of sparsity among loadings, and computational cost.

\smallskip
In this paper, we develop a fast deflation sparse PCA via subspace projections (SPCA-SP).
Similar to fast PCA \cite{abra2014},
we use randomized SVD algorithm to generate an initial subspace projection.
In addtion,
a series of extra subspace projections are constructed
by applying Household QR factorization to some auxiliary compound matrices.
These projections restrict the search space of each loading belonging to a low dimensional subspace,
while taking into account the orthogonality of the sparse loadings.
The corresponding construction process is quite technical.
The proposed approach belongs to a greedy algorithm based on postprocessing.
It mainly consists three alternative steps to find a sparse loading:
1) constructing the subspace projection;
2) searching an auxiliary low-dimensional PCA loading by using power method;
and
3) processing by truncating operation.

\smallskip
Our SPCA-SP method has the following merits:
1) Due to the introduction of subspace projections, the time cost of SPCA-SP could be very low even in high-dimensional cases.
2) Thanks to QR factorization, the computed sparse loadings are nearly orthogonal under small truncation.
3) Independent truncation for each loading tends to produce a balanced sparsity pattern.
Experimental results show that the developed method are comparable to other state-of-the-art  SPCA algorithms in quality,
while much more efficient in run time.

\smallskip
The remainder of the paper is organized as follows.
In Section 2, we introduce the basic ideas of deflation method and several truncation operators.
The proposed SPCA-SP method is presented in Section 3,
and an interesting connection between sparsity and orthogonality is also revealed.
Experiment results are provided in Section 4.
Finally, the conclusion is drawn in Section 5.

\section{Preliminaries}
Throughout the paper,
we use $ \| \cdot \| $ to denote the Euclidean norm of a vector,
$ \| \cdot \|_0 $ the count of nonzero entries,
$\| \cdot \|_F$  the Frobenius norm of a matrix.

\subsection{Deflation method for PCA}
We first introduce the deflation in the context of PCA.
Let $X \in \mathbb{R}^{n \times d}$ be a data matrix encoding $n$ samples and $d$ variables.
Without loss of generality, we assume that the variables contained in the columns of $ X $ are centered.
Let $ A_0=X^T X \in \mathbb{R}^{d \times d} $ denote the sample covariance matrix.

\smallskip
Deflation method aims to find $ r $ principal components by solving the following optimization problem sequentially:
\begin{align}
\label{DPCA}
  \mathbf{z}_t=\arg \max_{\mathbf{z}\in \mathbb{R}^d} \mathbf{z}^T A_{t-1} \mathbf{z},
  \quad
  \text{s.t.}\;
  \|\mathbf{z}\|=1,
\end{align}
for $  t=1,2,\ldots, r $.
The matrix $ A_{t} $ should be updated recursively  to eliminate the influence of the previous computed loading.
For instance, a widely used deflation formula is
\begin{align}
\label{DefinitionC}
  A_{t}= (I-\mathbf{z}_t \mathbf{z}_t^T ) A_{t-1} (I-\mathbf{z}_t \mathbf{z}_t^T ).
\end{align}
Note that the size of $A_{t} $ is fixed as $ d \times d  $ in each round,
which brings up a heavy cumulative workload in high dimensions.
In Section 3, we shall introduce some subspace projections to alleviate this problem.

\subsection{Truncation methods}
Given a principal component $ \mathbf{z}=(z_1,\ldots, z_d)^T \in \mathbb{R}^d  $,
it is a common way to employ an additional truncation operation to ensure sparsity \cite{hu2016sparse,mackey2009deflation,yuan2013truncated}.
In this paper,
we will use the truncation operator $ \mathcal{T}_\lambda (\cdot): \mathbb{R}^d \rightarrow  \mathbb{R}^d $,
which is one of the following three types.
\begin{itemize}
\item
 $\mathcal{T}_S$ (Truncation by Sparsity).
Given a cardinality $ 0 < \kappa_S < d $,
truncate the smallest $\kappa_S$ entries according to their absolute values.
The main advantage of $\mathcal{T}_S$ lies in its direct control of sparsity.

\item
$ \mathcal{T}_{E} $ (Truncation by Energy).
Sort the entries of $ \mathbf{z} $  in ascending order such that $ |\bar{z}_1| \leq |\bar{z}_2| \leq \ldots \leq |\bar{z}_d| $.
For a given real number  $ 0<\kappa_E < 1$,
choose $ i^* = \max  \{i\} $ with $ i $ satisfying $ \sum_{j=1}^i \bar{z}_j^2 \leq \kappa_E  \|\mathbf{z}\|^2 $.
Then truncate the smallest $ i^* $  entries, whose energy accounts for at most $ \kappa_E $ proportion,

\item $\mathcal{T}_{H}$ (Hard-Threshholding).
Given a threshold $ \kappa_H>0 $, set $\mathcal{T}_{H}(z_i)=0$  if $|z_i| < \kappa_H $, and $\mathcal{T}_{H}(z_i)=z_i$ otherwise.
\end{itemize}

\smallskip
For any $ \mathbf{z} \in \mathbb{R}^d $,
denote by $ s(\mathbf{z})=1-\|\mathbf{z}\|_0/d $  its sparsity.
For $\mathcal{T}_S$, it is trivial that  $\kappa_S/d \leq s(\mathbf{z}) <1 $.
For $ \mathcal{T}_{E} $, it was proved in \cite{hu2016sparse} that
\begin{align}
 \label{ThreKe}
 \lfloor\kappa_E d \rfloor/d \leq s(\mathbf{z}) \leq 1-1/d.
 \end{align}
For $ \mathcal{T}_{H} $, it was proved in \cite{hu2016sparse} that
\begin{align}
 \label{ThreKh}
 \begin{split}
  1-1/(d \kappa_H^2) \leq s(\mathbf{z}) \leq 1, \quad
  &\text{if} \; \kappa_H\geq 1/\sqrt{d};
 \\[5pt]
 0\leq s(\mathbf{z}) \leq 1-1/d , \quad
 &\text{if} \;\kappa_H<1/\sqrt{d}.
 \end{split}
\end{align}
Therefore, the truncation parameter can be used to control the sparsity of the loadings.

\section{SPCA via Subspace Projections}
In this section, we present our SPCA-SP algorithm.
The main contribution of the section is the construction of a series of subspace projections.
These projections will be used to restrict  the search space of each loading in a very low dimensional subspace
orthogonal to all previously computed sparse loadings.
We also find an interesting relationship between sparsity and orthogonality  after truncation operation.

\subsection{Sketch of SPCA-SP algorithm}
For $  t=1,2,\ldots, r $, we aim to find $\alpha_t $ sequentially such that
\begin{align}
\label{SDPCA}
 \alpha_t=\arg \max_{\alpha\in \mathbb{R}^m} \alpha^T ( P_{t-1}^T A_0 P_{t-1} ) \alpha,
\quad
  \text{s.t.}\;
  \|\alpha\|=1,
\end{align}
where $ \{P_{t-1} \in \mathbb{R}^{d \times m} \}_{t=1}^r $, $ m<\min\{n,d\} $, are subspace projections to be determined later.
In many practical fields, such as genomic analysis,
one can choose $ m \ll d $ to greatly save  the computational cost.

\smallskip
It is observable  that the vector $ P_{t-1} \alpha_t $, which belongs to the subspace spanned by the columns of the matrix $ P_{t-1} $,
is an approximation of principal component $ z_t $ in \eqref{DPCA}.
From the view point of rank-1 approximation,
\eqref{SDPCA} is identical to the following constrained optimization problem
\begin{align}
\label{EquiSDPCA}
 \min_{\alpha\in \mathbb{R}^m,\beta \in \mathbb{R}^n}\|XP_{t-1}-\beta \alpha^T\|_F^2,
 \quad
  \text{s.t.}\;
  \|\alpha\|=1,\; \|\beta\|=1.
\end{align}

\smallskip
In order to achieve a sparse loading, one could post process $ P_{t-1} \alpha_t $ by use of a truncation operation.
Specifically, for a given truncation operator $ \mathcal{T}_\lambda  $,  let
\begin{align}
\label{Truncation}
   \tilde{\mathbf{z}}_t=\frac{ \mathcal{T}_\lambda (P_{t-1} \alpha_t)}{\| \mathcal{T}_\lambda (P_{t-1} \alpha_t)\|},
   \quad
   \forall t\geq 1
\end{align}
be the corresponding sparse loading.
It should be noticed that the truncation parameter $ \lambda $  is a tuning parameter,
similar as the penalty weight in the penalized approaches.

\smallskip
It is worth emphasizing that the key difference between our SPCA-SP from most existing SPCA methods,
e.g. \cite{hu2016sparse,journee2010generalized,mackey2009deflation,yuan2013truncated},
is the introduction of an additional subspace projection in each round,
which intends to make the computation of each loading restricted in a low dimensional space,
and at the same time ensures the orthogonality of all sparse loadings.

\smallskip
Another noteworthy point is that  SPCA-SP is a postprocessing based algorithm without using sparsity penalization.
Sparsity penalized method is more commonly used in literature due to its distinct mathematical background .
But it seems impossible to apply such technique here,
because  $ P_{t-1}$  is not invertible
and thus commonly used block gradient descent technique \cite{wright2015} can not solve the corresponding optimization objective.

\smallskip
\begin{algorithm}[!htb]
    \caption{SPCA-SP deflation algorithm}
     {\bf Input:} Data matrix $X \in \mathbb{R}^{n \times d}$, number of sparse loadings  $ r $, subspace dimension $ m $,
      number of sampled rows $ c $,   truncation type  $\mathcal{T} $, and truncation parameter $ \lambda  $.
\begin{algorithmic}
     \For{$i=1,2,\ldots, n$}
        \State Generate probability $\xi_i = \|\mathbf{x}_{(i)}\|^2/ \|X\|_F^2$.
     \EndFor
     \For{$t=1,2,\ldots, c$}
        \State Sample  $ i_t \in \{1,\ldots,n\} $ with $\textbf{Pr}[i_t = \tau] = \xi_i, \: \tau = 1, \ldots, n$.
        \State Choose $ \mathbf{x}_{c,(t)} = \mathbf{x}_{(i_t)} / \sqrt{c \xi_{i_t}}$.
    \EndFor

    \State Compute $ X_c X_c^T $ and its SVD such that $X_c X_c^T =\sum_{j=1}^c \sigma_j^2  \mathbf{u}_j \mathbf{u}_j^T $.
    \State Set  $ P=[\mathbf{p}_1,\ldots,\mathbf{p}_m] $ with $ \mathbf{p}_i=(X_c^T \mathbf{u}_i )/\sigma_i $ .

    \For{$t=1,2,\ldots, r$}
        \State Compute the leading eigenvector $\alpha$ of $ P^T X^T X P$ by power method.
        \State Truncation: $ \tilde{\mathbf{z}}_t= \mathcal{T}_\lambda (P \alpha)/\|\mathcal{T}_\lambda (P \alpha)\| $.
        \State Construct the compound matrix $ B = [\tilde{\mathbf{z}}_1, \ldots, \mathbf{z}_t, P] $.
        \State Decompose  $ B=QR $ by Household QR factorization.
        \State Update  $ P=[\mathbf{q}_{t+1}, \ldots, \mathbf{q}_{t+m}]$;
    \EndFor
\end{algorithmic}
{\bf Output:} Sparse loadings $[\tilde{\mathbf{z}}_1, \ldots, \tilde{\mathbf{z}}_{r}] $.
\end{algorithm}

\subsection{Subspace projections}
In this section, we will introduce the construction of subspace projections in detail.
We  shall use the randomized SVD algorithm to construct an initial projection, just like fast PCA \cite{abra2014}.
After then,
we will employ a sequence of QR factorization to build the other subspace projections.
The related construction process is very technical.
It is worth to emphasize that the subspaces determined by these projections are orthogonal to the previously found sparse loadings.
It is well known that PCA loadings are orthogonal,
but as pointed out by \cite{hu2016sparse,mackey2009deflation},
this property is easily lost in SPCA.
Orthogonality is significant in SPCA because it ensures the independence of the physical meaning of the loadings,
thus further simplifying the interpretation.

\smallskip
Firstly, we turn to a fast SVD algorithm, named as LinearTimeSVD \cite{drineas2006fast},  to construct the projection $ P_0 $.
The purpose of using this algorithm is  to alleviate a part of time consuming.
 If $ n $ or $ d $ is not too large, exact SVD can also be used to construct the initial projection $ P_0 $.

\smallskip
For a given data matrix $X \in \mathbb{R}^{n \times d}$,
we generate a probability sequence as follows
\begin{align*}
    \xi_i = \displaystyle \frac{\mathbf{x}_{(i)}  \mathbf{x}_{(i)} ^T} { \|X\|_F^2}, \quad i = 1, 2, \ldots, n,
\end{align*}
where $ \mathbf{x}_{(i)}$ denotes the $i$-th row of $X$.
Let $ c $ be an integer satisfying $ c \leq \min \{n,d\} $.
For $ t=1,\ldots, c $ , sample  $ i_t \in \{1,\ldots,n\} $ with $\textbf{Pr}[i_t = \tau] = \xi_i, \: \tau = 1, \ldots, n$.
Let  $ X_c $ be a matrix of size $ c \times d  $,
the $t$-th row of which is determined by $ \mathbf{x}_{c,(t)} = \mathbf{x}_{(i_t)} / \sqrt{c \xi_{i_t}}$, $ t=1,\ldots, c $.
The singular value decomposition of  $  X_c X_c^T $ is denoted by
\begin{align*}
  X_c X_c^T =\sum_{j=1}^c \sigma_j^2  \mathbf{u}_j \mathbf{u}_j^T,
\end{align*}
where $\sigma_1  \geq \cdots \geq \sigma_c >0 $ are singular values of $ X_c $
and $ [\mathbf{u}_1,\ldots, \mathbf{u}_c]  $  forms an orthogonal matrix of size $c\times c$.
Choosing $ m \leq c $,  let
\begin{align*}
 \mathbf{p}_i = \displaystyle \frac{1}{\sigma_i}X_c^T \mathbf{u}_i \in \mathbb{R}^d, \quad  i=1,\ldots, m.
\end{align*}
The initial projection $ P_0 $  is then defined by
 \begin{align}
 \label{DefinitionP}
 P_0=[\mathbf{p}_1,\ldots,\mathbf{p}_m].
\end{align}
According to \cite{drineas2006fast},
the time complexity of the construction of $ P_0 $ is  $ O(nd+c^2d+c^3) $.

The sample size  $c$ and the subspace dimension $ m $ are free parameters.
For an expected cumulative percentage of explained variance $ 0<\texttt{CPEV}<1$,
one shall choose $c$ and $ m $  to satisfy
 \begin{align*}
 \frac{\mathbf{Tr}(P_0^T X^TX P_0) }{\mathbf{Tr}( X^TX)}> \texttt{CPEV},
\end{align*}
where $\mathbf{Tr}$ denotes the matrix trace.
This is not a difficult task when we are only interested in a few leading principal components.

\smallskip
Next, we are to construct the subspace projections $ P_t $, $t\geq 1 $ in a sequent manner,
based on the calculated sparse loadings.
We will employ QR factorization \cite{trefethen1997qr} in the construction process.
QR factorization decomposes the input matrix into the product of a square, orthogonal matrix $ Q $
and an upper triangular matrix $ R $.
It is usually used in solving linear systems of equations.
The QR decomposition of a matrix can be computed in different ways.
The use of Givens Rotations \cite{Givens1958} and
Householder reflections \cite{House1958} are two most commonly used and best known ones.
Here we prefer to use Householder version because it will be more efficient for the decomposition object defined below.

\smallskip
For any given matrix $ A \in \mathbb{R}^{n\times m} $,
a sequence of Householder reflections \cite{House1958} can be used to zero-out all the coefficients below the diagonal to compute its QR factorization:
\begin{align}
\label{HouseHold}
   H_mH_{m-1} \cdots H_1 A = R, \quad
   \text{ where}\;  H_mH_{m-1} \cdots H_1 = Q^T.
\end{align}
Each transformation $H_k$ annihilates the coefficients below the diagonal of column $ k  $
and modifies the coefficients in the trailing submatrix $ A(k:n,k+1:m) $.

\smallskip
Without loss of generality,
suppose that we have constructed the projection matrix $ P_{t-1}  $, $\forall t\geq 1$,
and obtained the corresponding sparse loadings $\tilde{\mathbf{z}}_1, \ldots, \tilde{\mathbf{z}}_t $.
To formulate the subsequent projection $ P_t $,
we introduce an auxiliary compound matrix as follows:
 $$
   B_t=[\tilde{\mathbf{z}}_1, \ldots, \tilde{\mathbf{z}}_t, P_{t-1} ] \in \mathbb{R}^{d \times (t+m)}.
$$
Applying Household QR factorization to the matrix $ B_t $ yields $  B_t= Q R $,
where $ Q=[\mathbf{q}_1,\ldots,\mathbf{q}_d] \in \mathbb{R}^{d \times d} $ is an orthogonal matrix.
Consequently,
the new subspace projection  $ P_t \in  \mathbb{R}^{d \times m}$  is constituted from a submatrix of $ Q $, i.e.,
\begin{align}
\label{PrjUpdate}
  P_t=[\mathbf{q}_{t+1}, \ldots, \mathbf{q}_{t+m}].
\end{align}
Noting that  $ Q^T B_t=R $ and $ R $ is an upper triangular matrix,
we immediately have
\begin{align}
\label{PrjOrt}
  P_t^T \tilde{\mathbf{z}}_i = \mathbf{0},
  \quad
  \forall 1\leq i \leq t.
\end{align}
The above property is desirable
because it means that the search space for the $(t+1)$-th untruncated loading
is orthogonal to all previously computed sparse loadings.
Therefore, we can anticipate that after small truncation, the sparse loadings are close to orthogonal.
This makes interpretation simpler.

\begin{remark}
In fact, there is no need to apply a complete QR factorization to matrix $ B_t $,
whose size will increase with $ t $.
Recall that $ B_{t}=[\tilde{\mathbf{z}}_1, \ldots, \tilde{\mathbf{z}}_{t-1}, \tilde{\mathbf{z}}_t, P_{t-1}] $,
where the first $(t-1)$ columns have already been treated in the previous steps.
Therefore, we only need to apply QR factorization for the submatrix $ [\tilde{\mathbf{z}}_t, P_{t-1}] \in \mathbb{R}^{d \times (m+1) }$.
The complexity of such decomposition is  $O(d(m+1)^2-(m+1)^3/3)$ \cite{trefethen1997qr}.
\end{remark}

\smallskip
\begin{remark}
In practice, $ t+m $ is usually far smaller than $ d $.
Even if $ t+m>d $, we could slightly modify $ P_t $ such that $ P_t=[\mathbf{q}_{t+1}, \ldots, \mathbf{q}_{\min\{d,t+m\}}] $.
\end{remark}

\subsection{Connection between orthogonality and sparsity}
In this section, we give a theoretical result about the connection
between orthogonality and sparsity after three truncation operations.
As far as we know, this relationship has never been revealed before.

\smallskip
First we introduce a notation to measure the orthogonality of two vectors.
Define
\begin{align*}
\langle \mathbf{a}, \mathbf{b}\rangle = 1-\frac{ |\mathbf{a}^T\mathbf{b}|}{ \|\mathbf{a}\|\|\mathbf{b}\|}.
\end{align*}
Observe that a larger $ \langle \mathbf{a}, \mathbf{b}\rangle $ means a better orthogonality.

\smallskip
The following lemma gives an upper bound of the orthogonality of two vectors after truncation.
\begin{lemma}
\label{SpaMeas}
    Let  $ \mathbf{a},\mathbf{b} $  be two unit orthogonal vectors in  $ \mathbb{R}^d $.
    For a given truncation operator $ \mathcal{T}_\lambda  $,
    let $ \mathbf{b}^+= \mathcal{T}_\lambda  (\mathbf{b}) $.
    Then
    \begin{align}
         \langle \mathbf{a},\mathbf{b}^+ \rangle \geq 1-\sqrt{1-\|\mathbf{b}^+\|^2}.
    \end{align}
\end{lemma}
\begin{proof}
Denote by $ \mathbf{sign}(\cdot) $ the sign function.
For any $ \mathbf{x} \in \mathbb{R}^d $, set
\begin{align*}
  & \mathbf{x}^+=(|\mathbf{sign}(b^+_1)| x_1,\ldots, |\mathbf{sign}(b^+_d)| x_d)^T,
  \\[5pt]
 & \mathbf{x}^-=\mathbf{x}-\mathbf{x}^+.
\end{align*}
It is trivial that $ (\mathbf{x}^+)^T \mathbf{x}^-=0 $.
Then
\begin{align}
\label{Ineq0}
  1-\langle \mathbf{a}, \mathbf{b}^+ \rangle
  =\frac{ |\mathbf{a}^T \mathbf{b}^+|}{\|\mathbf{b}^+\|}
  =\frac{ | (\mathbf{a}^+)^T \mathbf{b}^+|}{\|\mathbf{b}^+\|}
  \leq \max_{\|\mathbf{x}\|=1,\; \mathbf{x}^T\mathbf{b}=0} \frac{ | (\mathbf{x}^+)^T \mathbf{b}^+|}{\|\mathbf{b}^+\|}.
\end{align}
By Lagrange multiplier technique,
it is easy to see that the optimal solution to the right-hand side of \eqref{Ineq0} takes the form as
$ \mathbf{x}= k_1 \mathbf{b}^++ k_2 \mathbf{b}^- $, with $ k_1,k_2  $ being two constants.
The orthogonal constraint $ \mathbf{x}^T \mathbf{b}=0 $ implies that $ (\mathbf{x}^+)^T \mathbf{b}^+ +(\mathbf{x}^-)^T \mathbf{b}^-=0$,
where $ \mathbf{x}^+= k_1 \mathbf{b}^+$ and $\mathbf{x}^-= k_2 \mathbf{b}^- $.
Therefore
 \begin{align*}
 \|\mathbf{x}^+ \| \|\mathbf{b}^+\|&= | (\mathbf{x}^+)^T \mathbf{b}^+|=| (\mathbf{x}^-)^T \mathbf{b}^-|
 \\[5pt]
 &= \| \mathbf{x}^-\| \|\mathbf{b}^-\|
   =\sqrt{1-\|\mathbf{x}^+\|^2} \sqrt{1-\|\mathbf{b}^+\|^2},
\end{align*}
which implies that $ \|\mathbf{x}^+\|^2+\|\mathbf{b}^+\|^2 = 1$
This estimate together with \eqref{Ineq0} yields
\begin{align*}
 1-\langle \mathbf{a}, \mathbf{b}^+ \rangle \leq \| \mathbf{x}^+ \| = \sqrt {1-\|\mathbf{b}^+\|^2},
\end{align*}
which immediately yields the desired result.
\end{proof}

\smallskip
The next proposition gives the relationship between sparsity and orthogonality
after applying  the truncation operators introduced in Section 2.2.
\begin{proposition}
\label{SpaOrtho}
Let  $ \mathbf{a},\mathbf{b}\in \mathbb{R}^d $ be two unit orthogonal vectors.
Let $ \mathbf{b}^+= \mathcal{T}_\lambda  (\mathbf{b}) $,
where $ \mathcal{T}_\lambda $ is one of three  truncation operators in Section 2.2.
Then for $ \mathcal{T}_S$, when $ 0<\kappa_S < d $,
\begin{align}
\label{Ts}
         \langle \mathbf{a},\mathbf{b}^+ \rangle \geq 1-\sqrt{\kappa_S/d}.
\end{align}
For $ \mathcal{T}_E $, when $ 0<\kappa_E < 1$,
\begin{align}
\label{Te}
          \langle \mathbf{a},\mathbf{b}^+ \rangle  \geq 1-\sqrt{\kappa_E}.
\end{align}
For $ \mathcal{T}_E $, when $\kappa_H > 0 $,
\begin{align}
\label{Th}
           \langle \mathbf{a}, \mathbf{b}^+ \rangle \geq  1-\sqrt{1-\|\mathbf{b}^+\|_0 \kappa_H^2}.
\end{align}
Here $ \kappa_S $, $ \kappa_E $ and $ \kappa_H $ are corresponding truncation parameters, respectively.
\end{proposition}

\begin{proof}
Set $ \mathbf{b}^-=\mathbf{b}-\mathbf{b}^+ $. According to Lemma \ref{SpaMeas},
\begin{align}
\label{EstT1}
         \langle \mathbf{a},\mathbf{b}^+ \rangle \geq 1-\sqrt{1-\|\mathbf{b}^+\|^2}=1-\|\mathbf{b}^-\|.
\end{align}

\smallskip
If $\mathcal{T}_S$ is used, then
\begin{align*}
     \frac{1-\|\mathbf{b}^-\|^2}{d-\kappa_S} \geq \frac{\|\mathbf{b}^-\|^2}{ \kappa_S },
\end{align*}
which implies that $ \|\mathbf{b}^-\|^2 \leq \kappa_S/d $. Thus \eqref{Ts} holds well.

\smallskip
If $\mathcal{T}_{E}$ is used,
then the desired result $ \eqref{Te} $ follows immediately from the fact that $\|\mathbf{b}^-\|^2 \leq \kappa_E$.

\smallskip
If $\mathcal{T}_{H}$ is used, then $  \|\mathbf{b}^+\|^2 \geq \|\mathbf{b}^+\|_0 \kappa_H^2 $.
Applying this estimate in \eqref{EstT1} yields the desired result $ \eqref{Th} $.
\end{proof}
In view of Proposition \ref{SpaOrtho} and \eqref{PrjOrt},
we can use truncation parameters to control the orthogonality performance of the sparse loadings.

\section{Experiments}
In order to evaluate the effectiveness of the proposed SPSCA-SP algorithm,
we conduct experiments on four data sets:
a synthetic data with some underlying sparse loadings \cite{zou2006sparse};
classical Pitprops data \cite{jeffer1967pitprops};
Gene data with high dimension and small sample size \cite{golub1999gene};
and a set of random data with increasing dimensions.

We compare our SPSCA-SP with several baseline algorithms,
including SPCA  \cite{zou2006sparse}, PathSPCA \cite{daspremont2007full}, Tpower \cite{yuan2013truncated} and SPCArt\cite{hu2016sparse}.
We programme Tpower, SPCArt and SPCA-SP in Python.
The results of SPCA and PathSPCA come directly from the references.
We are mainly interested in following criteria.
\begin{itemize}
\item
\textbf{Cumulative proportion of explained variance} is defined by
$$
    \texttt{CPEV}=\frac{\mathbf{Tr}(W^T X^T X W)}{\mathbf{Tr}(X^T X)},
$$
where $W=[\textbf{w}_1, \ldots, \textbf{w}_r] \in \mathbb{R}^{d\times r}$
is a set of unit orthogonal basis of the space $\texttt{span}\{\mathbf{\tilde{z}}_1, \ldots, \mathbf{\tilde{z}}_r\} $.
\item
\textbf{Orthogonality}. Given a loading matrix $ Z=[\mathbf{\tilde{z}}_1, \ldots, \mathbf{\tilde{z}}_r]\in \mathbb{R}^{d\times r} $, we use
$$
   1-\frac{|Z^TZ|-\mathbf{Tr}(Z^TZ)}{r(r-1)}
$$
to measure the total orthogonality,
where $|\cdot|$ denotes the sum of the absolute values of all entries of a matrix.

\item
Denote by \textbf{NZ} the total number of non-zeros in loadings.
Let $ \textbf{SP}=1-\textbf{NZ}/(rd) $ denote the total sparsity.
\textbf{Loading pattern} describes the balance of sparsity among the loadings.
For example, 3-3-3-3-3-3 means that the number of non-zeros in each loading is 3.
As pointed out by \cite{hu2016sparse}, a quite few existing algorithms yield unreasonable sparsity patterns
 such that highly dense leading loadings close to those of PCA, while the minor ones are sparse.

\item
\textbf{CPU time} measures the running time of the algorithms.
\end{itemize}

\subsection{Synthetic Data}
In this section, we test whether SPCA-SP can recover some underlying sparse loadings
of the synthetic data introduced in \cite{zou2006sparse},
which include three hidden Gaussian factors
\begin{align*}
  h_1 \sim \mathcal{N} (0, 290),
  \quad
   h_2 \sim \mathcal{N} (0, 300),
  \\[5pt]
  h_3  = -0.3 h_1 + 0.925 h_2 + \epsilon,
 \quad\epsilon \sim \mathcal{N}(0, 1).
\end{align*}
Then 10 observable variables are generated by
\begin{align*}
d_i = h_1 + \epsilon_i^1, \quad \epsilon_i^1 \sim \mathcal{N}(0, 1),  \quad  &i = 1, 2, 3, 4, \\
d_i = h_2 + \epsilon_i^2, \quad  \epsilon_i^2 \sim \mathcal{N}(0, 1), \quad  &i = 5, 6, 7, 8, \\
d_i = h_3 + \epsilon_i^3, \quad  \epsilon_i^3 \sim \mathcal{N}(0, 1), \quad  &i = 9, 10.
\end{align*}
We consider $ r=2 $ since the first two principal components  together explain 99.6\% of the total variance.
Note that $h_1$ and $h_2$ are independent,
while $h_3$ is correlated with both of them but more dependent on $h_2$.
The most acceptable two sparse loading patterns are $ 1-4 $, $9-10$; $ 5-10$ and $ 1-4 $; $ 5-10$.
\begin{table}[!htb]
\tbl{Recovering sparse loadings of syntectic data ($r=2$) }
      {\begin{tabular}{|c|cc|cc|cc|}
        \hline
        & \multicolumn{2}{c}{$\mathcal{T}_S$ ($\kappa_S=4 $)} & \multicolumn{2}{|c}{$\mathcal{T}_E$  ($\kappa_E=0.2$)} &
        \multicolumn{2}{|c|}{$\mathcal{T}_H$  ($\kappa_H=1/\sqrt{d}$)}   \\
        \hline
         \hline
                  & $\mathbf{\tilde{z}}_1$ &$ \mathbf{\tilde{z}}_2$
                  & $\mathbf{\tilde{z}}_1$ &$ \mathbf{\tilde{z}}_2$
                  & $\mathbf{\tilde{z}}_1$ &$ \mathbf{\tilde{z}}_2$    \\
        \hline
$d_{1  }$   &    0.0000     & 0.4952     &    0.0000     & 0.5000     &    0.0000     & 0.5000    \\
$d_{2  }$   &    0.0000     & 0.4952     &    0.0000     & 0.5000     &    0.0000     & 0.5000    \\
$d_{3  }$   &    0.0000     & 0.4952     &    0.0000     & 0.5000     &    0.0000     & 0.5000    \\
$d_{4  }$   &    0.0000     & 0.4952     &    0.0000     & 0.5000     &    0.0000     & 0.5000    \\
$d_{5  }$   &    0.4057     & 0.0000     &    0.4058     & 0.0000     &   0.4058      & 0.0000     \\
$d_{6  }$   &    0.4058     & 0.0000     &    0.4058     & 0.0000     &   0.4057      & 0.0000     \\
$d_{7  }$   &    0.4057     & 0.0000     &    0.4057     & 0.0000     &   0.4058      & 0.0000     \\
$d_{8  }$   &    0.4057     & 0.0000     &    0.4058     & 0.0000     &   0.4057      & 0.0000     \\
$d_{9  }$   &    0.4132     & 0.0978     &    0.4139     & 0.0000     &   0.4125      & 0.0000     \\
$d_{10 }$  &    0.4132     & 0.0978     &    0.4125     & 0.0000     &  0.4139       & 0.0000     \\
        \hline
         CPEV   & \multicolumn{2}{c}{0.9943} & \multicolumn{2}{|c}{0.9840} &
         \multicolumn{2}{|c|}{0.9840} \\
        \hline
    \end{tabular}}
    \label{Synsp}
\end{table}
We take $ c=5 $ and  $ m =3 $ as an example.
The computed sparse loadings based on three types of truncation operations are listed in Table \ref{Synsp}.
It is obvious that SPCA-SP successfully recovers the desirable sparse loading patterns.

\subsection{Pitprops data}
Pitprops data, which contains 180 observations and 13 features,
is a standard benchmark to evaluate the performance  of SPCA algorithms.
The first six principal components explain 86.9\% variance of the data.
The algorithms are tested to find  $ r=6$  sparse loadings.
We directly cite the best results reported in  \cite{hu2016sparse} for the baseline algorithms.
Set  $ c=11 $ and $ m=5 $.
The truncation parameters in SPCA-SP are tuned to yield the expected loading patterns.

\smallskip
It is  observed from Table \ref{tableCP} that
SPCA-SP has achieved the competitive empirical performance with these baseline algorithms.
Especially, in the case of balanced loading pattern 3-3-3-3-3-3,
SPCA-SP ($\mathcal{T}_{S}, \kappa_S = 10$)  obtains the maximum explained variance,
while its orthogonality is better than that of all other algorithms except SPCA.
Furthermore,  SPCA-SP ($\mathcal{T}_{H}, \kappa_H=0.4$)  outperforms classical SPCA under a more sparse loading mode.

\begin{table}[!hb!t]
\tbl{Comparison of SPCA-SP with four baseline methods on Pitprops data ($r=6$).}
    {\begin{tabular}{|c|c|c|c|c|}
        \hline
        Algorithms & NZ & Loading Pattern & Orthogonality &   CPEV \\
        \hline
        \hline
        SPCA       & 18 & 3-3-3-3-3-3 & 0.9905 & 0.7727 \\
        PathSPCA & 18 & 3-3-3-3-3-3 & 0.9516 & 0.7840 \\
        SPCArt     & 18 & 3-3-3-3-3-3 & 0.9572 & 0.7514 \\
        Tpower     & 18 & 3-3-3-3-3-3 & 0.9545 & 0.7819 \\
        SPCA-SP ($ \kappa_{S}=10 $) & 18 & 3-3-3-3-3-3 & 0.9576 & 0.7865\\
        \hline
        SPCArt                                        & 18 & 4-2-4-3-3-2 & 0.9819 & 0.8013 \\
        SPCA-SP ($ \kappa_{H}=0.35 $)  & 17 & 5-2-4-2-2-2 & 0.9643 & 0.8056 \\
        SPCA-SP ($ \kappa_{E}=0.4$)     & 13 & 3-3-2-2-2-1 & 1.0000 & 0.7765 \\
        \hline
    \end{tabular}}
    \label{tableCP}
\end{table}

\subsection{Gene data}
In this section, we consider high-dimensional Gene data \cite{golub1999gene}, which contains 72 samples with 7129 variables,
Since $ n\ll d$, we use an exact SVD algorithm to generate the initial subspace projection for SPCA-SP.

\smallskip
We first examine how the performance of SPCA-SP is affected by the subspace dimension when $ 2\leq m \leq 72 $.
For simplicity,
we fix truncation parameters  as $ \kappa_S =5500 $, $ \kappa_E=0.01 $ and $\kappa_H=1/300 $.
As shown in Figure \ref{genero},
when  $ m $ increases from 2 to 6,
explained variance, sparsity and orthogonality all increase rapidly.
After then the performance on three criteria becomes stable, almost unchanged.
On the other hand, the time cost goes approximately linearly with the subspace dimension.
These observations demonstrate that SPCA-SP is robust and stable with respect to the subspace dimension $ m $.
The observation results also show that in practical application,
a proper small $ m $  is enough to ensure a good tradeoff between explained variance, sparsity, orthogonality and computational speed.

\begin{figure}[!thb]
\label{Robust}
    \centering
     \scalebox{0.78}{\includegraphics{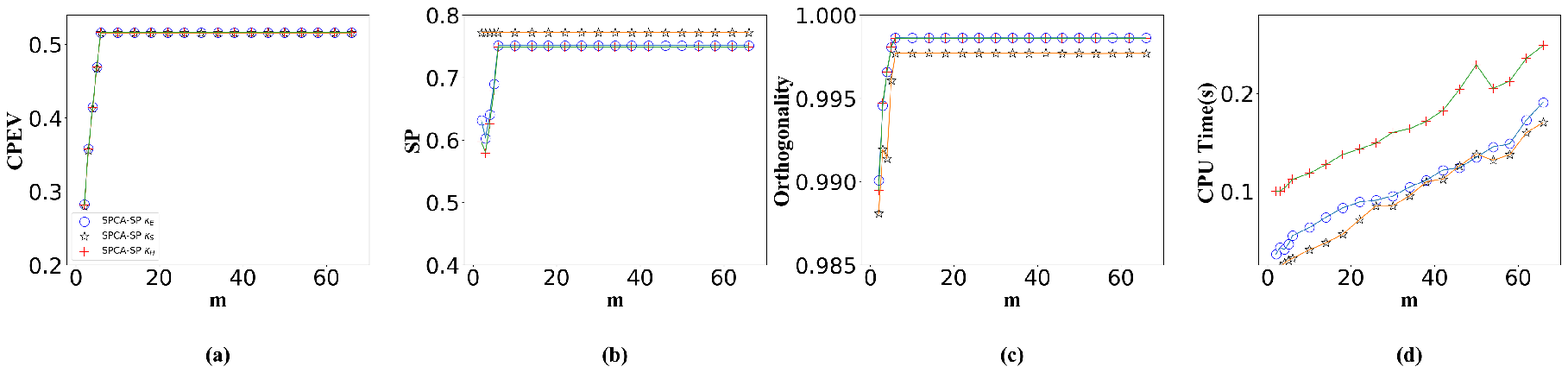}}
    \caption{Performance of SPCA-SP on gene data with increasing subspace dimension $ m $  ($r=6$). (a) CPEV. (b) Sparsity. (c) Orthogonality. (d) CPU time.}
    \label{genero}
\end{figure}

 \begin{figure}[!thb]
\label{SameSpGene}
    \centering
    \scalebox{0.75}{\includegraphics{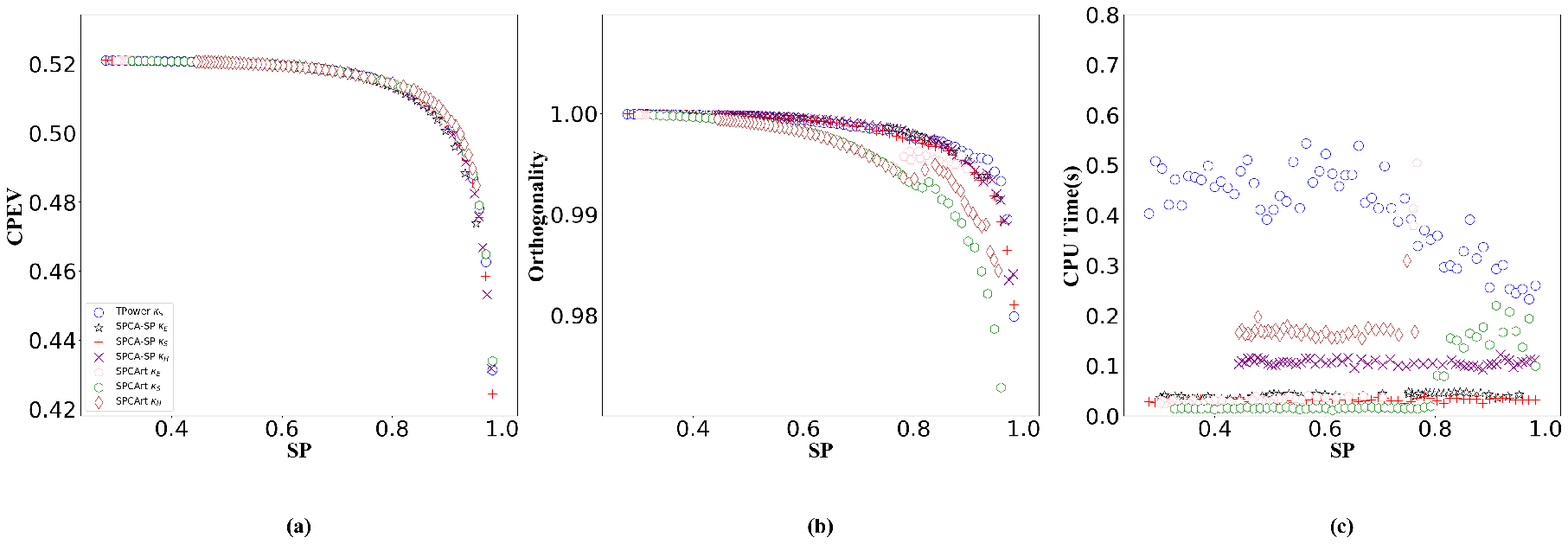}}
    \caption{SPCA-SP ($m=6$) versus Tpower and SPCArt on gene data ($r=6$).  (a) CPEV. (b) Orthogonality. (c) CPU time. }
    \label{genesp}
\end{figure}

\smallskip
Next, we choose Tpower and SPCArt algorithms as our comparison objects.
As shown in \cite{hu2016sparse,yuan2013truncated},
both methods are among the top fast SPCA solvers.
Specifically, Tpower is an iterative power method along with $\mathcal{T}_S $ truncation,
while SPCArt belongs to a block method alternatively rotating the PCA basis and truncating small entries.

\smallskip
We run the algorithms on a range of truncation parameters.
The subspace dimension in SPCA-SP is kept as $ m=6 $.
The results of CPEV, orthogonality and computational time under the same sparsity are depicted in Figure \ref{genesp}.
It is obvious from Figure \ref{genesp} (a) and (b) that Tpower, SPCArt and SPCA-SP perform similarly on explained variance and orthogonality.
But from Figure \ref{genesp} (c), one can find that only the time consuming of SPCA-SP is  stable with respect to sparsity.
Furthermore, due to the introduction of subspace projections,
our SPCA-SP method achieves the best performance on the computational speed as expected.

\subsection{Random data}
In this section, we consider random data sets with increasing dimensions.
As in \cite{hu2016sparse,journee2010generalized},
we first consider zero-mean, unit variance Gaussian data with $ d=100, 400,700,1000,1300 $,  and take  $ n=d+1 $.
We still compare SPCA-SP with Tpower and SPCArt methods.
For simplicity, we use the truncation $ \mathcal {T}_S $ for three methods.
Once the truncation parameter $\kappa_S $ is fixed,  all three methods have the same sparsity.
We take $ r=20 $ and $ \kappa_S=[0.7d] $ as an example.
In SPCA-SP, let $c = d/2 $ and $ m=80 $.
It is observed from Figure \ref{Random1}  that the performance of all three methods are similar on explained variance and orthogonality.
But, as reflected in Figure \ref{Random1} (c),  the time cost of Tpower and SPCArt increases nonlinearly with the data dimension,
while that of SPCA-SP is lower and goes almost linearly.

\begin{figure}[!thb]
    \centering
    \scalebox{0.75}{\includegraphics{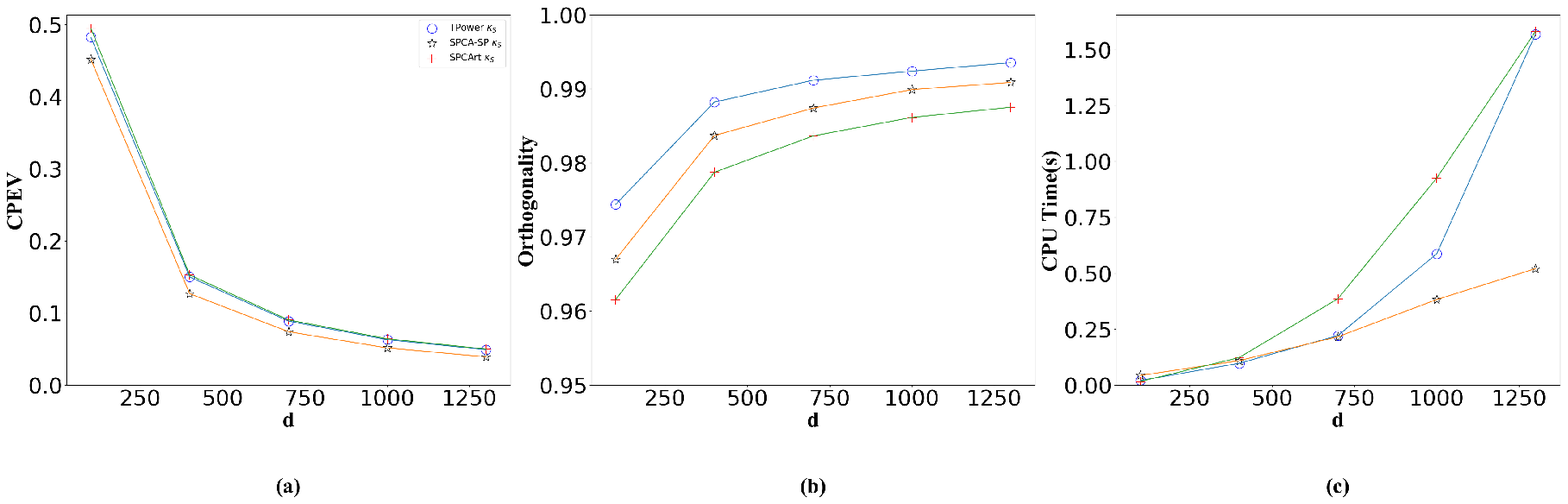}}
    \caption{SPCA-SP versus Tpower and SPCArt on random data ($n\simeq d $) with increasing data dimension (a) CPEV. (b) Orthogonality. (c) CPU time. }
    \label{Random1}
\end{figure}

Next we consider an extra high-dimensional case with $ n \ll d $.
We  fix $ n=500 $ and consider zero-mean, unit variance Gaussian data with $d=1000, 4000, 10000, 20000, 30000 $.
An exact SVD algorithm is employed to yield the initial subspace projection.
The subspace dimension is fixed as $ m=30 $.
We take $ r=20 $ and $ \kappa_S=[0.85d] $ as an example.
From Figure  \eqref{Random2},
we find the time consuming of SPCA-SP for such extra high-dimensional data is much lower as compared to Tpower and SPCArt.

\begin{figure}[!thb]
    \centering
    \scalebox{0.75}{\includegraphics{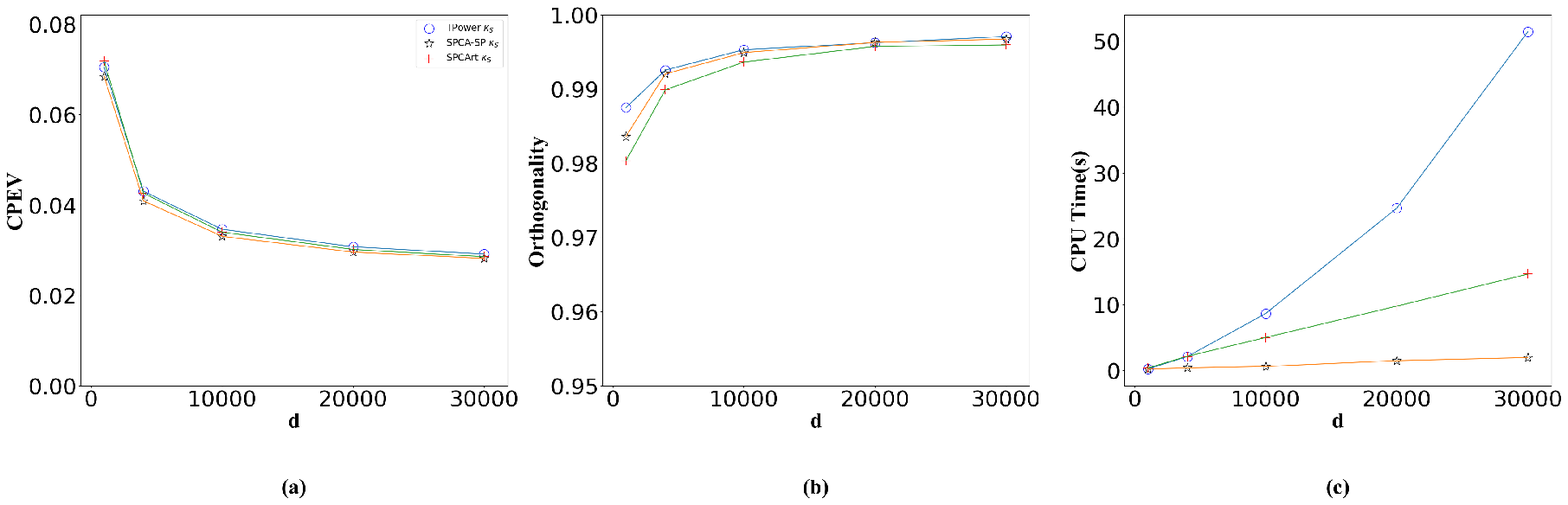}}
    \caption{SPCA-SP versus Tpower and SPCArt on random data ($n \ll d $) with increasing data dimension (a) CPEV. (b) Orthogonality. (c) CPU time. }
    \label{Random2}
\end{figure}

\smallskip
The results from Figure \ref{Random1} and Figure \ref{Random2} demonstrate that
our SPCA-SP is  highly efficient in processing high-dimensional data.

\section{Concluding Remarks}
Most conventional SPCA approaches are time consuming when dealing with high-dimensional data.
With the aid of QR factorization and certain compound matrices,
a series of subspace projections were  developed in this paper to enable a fast SPCA.
Because of its simplicity and efficiency,
Household reflection is preferred in the decomposition process,
and thus only one small submatrix needs to be processed in each round.

\smallskip
The proposed projections restrict the calculation of each loading
in a very low dimensional space, while taking into account the orthogonality of the sparse loadings.
These two characteristics make the SPCA-SP method achieve a good tradeoff between
sparsity, orthogonality, explained variance, balance of sparsity among loadings, and computational cost.
The comparative results in the previous section indicate that
SPCA-SP is an attractive and practical one in handling high-dimensional data and looking for many loadings.

\section*{Funding}
This work is supported in part by the National Natural Science Foundation of
China under grants 11771257, by the Shandong Province Natural Science Foundation under grant ZR2018MA008.

\end{document}